\documentclass[letterpaper, 10 pt, conference]{cls/ieeeconf}
\IEEEoverridecommandlockouts
\usepackage[table,usenames,dvipsnames]{xcolor}      
\usepackage[noadjust]{cite}
\usepackage{amsmath,amssymb,amsfonts,amsthm,dsfont,mathtools}
\usepackage{nicematrix}

\usepackage{adjustbox}
\usepackage{graphicx,tabularx,adjustbox,color}
\usepackage{multirow}
\usepackage[font=footnotesize]{caption}
\usepackage[font=footnotesize]{subcaption}

\usepackage{algorithmic}
\usepackage[ruled,vlined]{algorithm2e}
\usepackage{stackengine}
\def\delequal{\mathrel{\ensurestackMath{\stackon[1pt]{=}{\scriptstyle\Delta}}}}
\usepackage{dirtytalk}
\allowdisplaybreaks
\renewcommand{\baselinestretch}{.985}

\usepackage[breaklinks=true, colorlinks, bookmarks=true, citecolor=Black, urlcolor=Violet,linkcolor=Black]{hyperref}

\DeclareMathOperator*{\diag}{diag}

\newtheorem{theorem}{Theorem}[section]
\newtheorem{proposition}[theorem]{Proposition}

\theoremstyle{remark}
\newtheorem{remark}[theorem]{Remark}
\theoremstyle{definition}
\newtheorem{definition}[theorem]{Definition}
\newtheorem{problem}{Problem}

\newcommand{\calA}{{\cal A}}

\newcommand{\calH}{{\cal H}}

\newcommand{\calK}{{\cal K}}
\newcommand{\calL}{{\cal L}}

\newcommand{\calO}{{\cal O}}

\newcommand{\calU}{{\cal U}}

\newcommand{\calX}{{\cal X}}

\newcommand{\calZ}{{\cal Z}}

\newcommand{\bfd}{\mathbf{d}}

\newcommand{\bfk}{\mathbf{k}}

\newcommand{\bfp}{\mathbf{p}}
\newcommand{\bfq}{\mathbf{q}}

\newcommand{\bfu}{\mathbf{u}}

\newcommand{\bfx}{\mathbf{x}}
\newcommand{\bfy}{\mathbf{y}}
\newcommand{\bfz}{\mathbf{z}}

\newcommand{\bftheta}{\boldsymbol{\theta}}

\newcommand{\bfA}{\mathbf{A}}
\newcommand{\bfB}{\mathbf{B}}
\newcommand{\bfC}{\mathbf{C}}
\newcommand{\bfD}{\mathbf{D}}

\newcommand{\bfI}{\mathbf{I}}

\newcommand{\bfX}{\mathbf{X}}

\newcommand{\bbE}{\mathbb{E}}

\newcommand{\bbR}{\mathbb{R}}

\newcommand{\ubfu}{\underline{\bfu}}
\newcommand{\Var}{\textit{Var}}
\newcommand{\Cov}{\textit{Cov}}

\title{\LARGE \bf Safe Control Synthesis with Uncertain Dynamics and Constraints} 

\author{Kehan Long$^{1}$ \quad Vikas Dhiman$^{2}$  \quad Melvin Leok$^{1}$ \quad Jorge Cort{\'e}s$^{1}$ \quad Nikolay Atanasov$^{1}$%
\thanks{We gratefully acknowledge support from NSF RI IIS-2007141.}%
\thanks{$^{1}$The authors are with the Contextual Robotics Institute, University of California San Diego, La Jolla, CA 92093, USA. {\tt\small \{k3long,mleok,cortes,natanasov\}@ucsd.edu}.}%
\thanks{$^{2}$V. Dhiman is with Department of Electrical and Computer Engineering, University of Maine, Bangor, ME 04469, USA. {\tt\small \{vikas.dhiman\}@maine.edu}.}%
}

\begin{document}
\maketitle
\begin{abstract}
This paper considers safe control synthesis for dynamical systems with either probabilistic or worst-case uncertainty in both the dynamics model and the safety constraints. We formulate novel probabilistic and robust (worst-case) control Lyapunov function (CLF) and control barrier function (CBF) constraints that take into account the effect of uncertainty in either case. We show that either the probabilistic or the robust (worst-case) formulation leads to a second-order cone program (SOCP), which enables efficient safe and stable control synthesis.
We evaluate our approach in PyBullet simulations of an autonomous robot navigating in unknown environments and compare the performance with a baseline CLF-CBF quadratic programming approach.
\end{abstract}

\section{Introduction}
\label{sec: intro}

Autonomous robotic systems are increasingly employed in warehouse and home automation, transportation, and security applications. A crucial aspect of successfully deploying such systems is the satisfaction of safety and stability requirements, even in the presence of uncertainty in the system model or constraints. The notion of safety in the context of program correctness was first introduced in the $1970$'s \cite{lamport1977safety}. Around the same time, Artstein \cite{Artstein1983StabilizationWR} introduced control Lyapunov functions (CLFs) to enforce stability in the context of nonlinear system control. The seminal work of Sontag \cite{SONTAG1989117} established a universal formula for constructing feedback control laws that stabilize nonlinear systems. In the $2000$'s, barrier certificates were proposed to formally prove the safety of closed-loop nonlinear and hybrid systems \cite{barrier-certificate, barrier-hybrid}. Control barrier functions (CBFs) were developed to support task-independent safe control synthesis, serving as a barrier certificate for a closed-loop nonlinear system \cite{wieland2007}.

A key observation is that, for control-affine systems, the CLF and CBF conditions are linear in the control input, allowing a formulation of safe and stable control synthesis as a quadratic program (QP) \cite{ames2016control, nguyen2016acc, cbf}. 
CLF-CBF-QP techniques have been successfully employed in a variety systems, including aerial robots \cite{Wang2017SafeCM}, walking robots \cite{nguyen2016cdc}, and automotive systems \cite{xu2017realizing}. Most existing work, however, assumes complete knowledge of the system dynamics and control barrier functions. In reality, the dynamics model and safety constraints are obtained using noisy sensor data and simplifying assumptions, leading to uncertainty and errors that should be captured when ensuring safety and stability.

Capturing system-model and barrier-function estimation errors impacts the formulation of CLF and CBF constraints, and no longer give rise to QPs. Our main contribution is to show that such uncertainty-aware stability and safety constraints can still be formulated as convex constraints under two different models of uncertainty: probabilistic and worst-case.
To capture probabilistic uncertainty, we specifically consider \emph{Gaussian Process} (GP) regression as an example approach for modeling a probability distribution over a function space. 
When the estimated barrier function and system dynamics are described by a GP, we aim to ensure probabilistic safety and stability up to a user-specified risk tolerance. We compute the distribution of the CLF and CBF constraints, and use Cantelli's inequality \cite{CantelliSuiCD} to bound the computed means with a margin dependent on the variances and the desired risk-tolerance. The control input appears linearly in the mean and quadratically in the variance of the CLF and CBF constraints. This allows us to restate the probabilistic constraints as second-order cone constraints, leading to a second-order cone program (SOCP), which is convex and can be solved efficiently online. 

When \emph{worst-case error bounds} on the system dynamics, barrier function and its gradient are given, we formulate a robust safe control
synthesis problem. Under worst-case disturbances, we show that the input appears both linearly and within a norm term in the CLF and CBF constraints. Like the probabilistic formulation, the original QP problem can be reformulated as a convex SOCP for safe control synthesis. 

We demonstrate our safe control synthesis techniques in mobile robot navigation simulations. We consider a robot tasked to follow a desired path in an unknown environment, relying on online noisy obstacle sensing and offline dynamic model estimation to ensure safety and stability. We show that both the probabilistic and the robust CLF-CBF-SOCP formulation allows the robot to safely track the deisred path.


In summary, we make the following \textbf{contributions}. First, we formulate novel probabilistic safety and stability constraints by considering stochastic uncertainty in the barrier functions and system dynamics. Second, we formulate novel robust safety and stability constraints by considering worst-case error bounds in the barrier functions and system dynamics. Finally, we show that either the probabilistic or the worst-case formulations lead to a (convex) SOCP, enabling efficient synthesis of safe and stable control.

\section{Related Work}
\label{sec: related_work}

This section reviews recent works on safe control synthesis that address uncertainty due to unmodeled dynamics, input disturbances, and barrier function estimation.


Jankovic \cite{jankovic_robust_2018} considers worst-case disturbance bounds on the system dynamics and proposes robust CBF formulations. Eman \textit{et al.} \cite{yousef2019cdc} utilize convex hulls to model disturbances in a CBF-based safety framework. 
Clark \cite{andrew2019acc} considers stochastic control systems with incomplete information and derives sufficient conditions for ensuring safety on average. 
Nguyen and Sreenath \cite{Nguyen2021} formulate a robust CLF-CBF QP by introducing robust constraints to guarantee stability and safety under model uncertainty. Hewing \textit{et al.} \cite{lukas_2020_TCST} present a model predictive control (MPC) approach that integrates a nominal system with a residual part modeled as a GP. Compared to our formulation, this approach enables optimizing the control performance over a longer future horizon but requires time discretization and convexification of the safety constraints. In contrast, our formulations operate in continuous time and handle general safe set descriptions. Ahmadi \textit{et al.} \cite{Ahmadi2020RiskSensitivePP} introduce a conditional value-at-risk (CVaR) barrier function to ensure safety for systems with stochastic uncertainty. The approach guarantees safety with high probability even for worst-case scenarios but the computation cost is high and the formulation is restricted to linear systems. Our approach enables efficient control synthesis for general control-affine systems. Another line of research formulates safe control synthesis as trajectory optimization. Alcan and Kyrki \cite{Alcan_DDP_RAL} employ differential dynamics programming (DDP) to enforce safety under additive uncertainty. In \cite{Hassan_DBaS_RAL}, the DDP idea is combined with CBF to introduce a barrier state formulation for safety of discrete-time systems.

\emph{Input-to-state safety} (ISSf) was introduced in \cite{romdlony2016cdc} to handle input disturbances and was used in \cite{Kolathaya2019issf} to enlarge a safe set by modifying a CBF. Alan \textit{et al.} \cite{alan2021safe} introduce a tunable ISSf-CBF for safe control synthesis while reducing conservatism. Cosner \textit{et al.} in \cite{cosner2021measurement} introduce measurement-robust CBFs to account for uncertainty in state estimation and conduct experiments on a Segway. 

Srinivasan \textit{et al.} \cite{srinivasan2020synthesis} estimate barrier functions online using a Support Vector Machine and solve a CLF-CBF QP to generate safe control inputs. Zhang \textit{et al.} \cite{zhang2021adversarially} construct robust output CBFs from safe expert demonstrations while considering worst-case error bounds in the measurement map and system dynamics. 



This paper unifies and extends our prior work \cite{dhiman2020control, Long_learningcbf_ral21} by considering safe control synthesis with uncertainty in the system dynamics and the barrier function simultaneously and studying two separate cases of probabilistic and worst-case uncertainty. In contrast, \cite{dhiman2020control} only considered probabilistic uncertainty in the dynamics using Gaussian process regression, while \cite{Long_learningcbf_ral21} only considered worst-case error bounds in the barrier function. We show that in either case the safe control synthesis problem is a convex SOCP, which enables efficient safe and stable control synthesis online. 


\section{Problem Formulation}\label{sec:problem}


Consider a robot with dynamics model:
\begin{equation}\label{eq: dynamic}
\begin{aligned}
    \dot{\bfx} = f(\bfx) + g(\bfx) \bfu &=     [f(\bfx) \; g(\bfx)] \cdot\begin{bmatrix}
    1 \\
    \bfu
    \end{bmatrix} \delequal F(\bfx)\ubfu,
\end{aligned}
\end{equation}
where $\bfx \in \calX \subseteq \mathbb{R}^{n}$ is the robot state and $\ubfu \in \underline{\mathcal{U}} = \{1\} \times \mathbb{R}^{m}$ is the control input.\footnotemark{} We assume $f : \mathbb{R}^{n} \mapsto \mathbb{R}^{n}$ and $g : \mathbb{R}^{n} \mapsto \mathbb{R}^{n \times m}$ are continuously differentiable.



\begin{definition}
A continuously differentiable function $V: \mathbb{R}^n \mapsto {\mathbb{R}_{\geq 0}}$ is a \emph{control Lyapunov function (CLF)} for the system \eqref{eq: dynamic} if there exists a class $\mathcal{K}$ function $\alpha_V$ such that:
\begin{equation}
    \inf_{\ubfu \in \underline{\mathcal{U}}} \textit{CLC}(\bfx,\ubfu) \leq 0, \quad \forall \bfx \in \calX,
\end{equation}
where the \emph{control Lyapunov condition (CLC)} is:
\begin{equation}\label{eq:clc_define}
\begin{aligned}
    \textit{CLC}(\bfx,\ubfu) &\delequal \mathcal{L}_f V(\bfx) + \mathcal{L}_g V(\bfx)\bfu + \alpha_V( V(\bfx)) \\
    & =  [\nabla_{\bfx} V(\bfx)]^\top F(\bfx)\ubfu + \alpha_{V}(V(\bfx)).
\end{aligned}
\end{equation}
\end{definition}

A CLF $V$ may be used to encode a variety of control objectives, including path following \cite{Long_learningcbf_ral21}, adaptive cruise control \cite{xu2017realizing}, and bipedal robot walking \cite{nguyen2016cdc}.

\footnotetext{\textbf{Notation}: We denote by $\bfI_n \in \bbR^{n \times n}$ the identity matrix and $\partial \calA$ the boundary of a set $\calA \subset \mathbb{R}^n$. For a vector $\bfx$ and a matrix $\bfX$, we use $\|\bfx\|$ and $\| \bfX\|$ to denote the Euclidean norm and the spectral norm. We use $\text{vec}(\bfX) \in \mathbb{R}^{nm}$ to denote the vectorization of $\bfX \in \mathbb{R}^{n \times m}$, obtained by stacking its columns. 
We denote by $\nabla$ the gradient and $\calL_fV  = \nabla V \cdot f$ the Lie derivative of a differentiable function $V$ along a vector field $f$. We use $\otimes$ to denote the Kronecker product and $\mathcal{GP}(\mu(\bfx), K(\bfx,\bfx'))$ to denote a Gaussian Process distribution with mean function $\mu(\bfx)$ and covariance function $K(\bfx,\bfx')$. A continuous function $\alpha: [0,a)\rightarrow [0,\infty )$ is of class $\calK$ if it is strictly increasing and $\alpha(0) = 0$, and it is of class $\calK_{\infty}$ and $\lim_{r \rightarrow \infty} \alpha(r) = \infty$.
}

To define safety requirements for the control objective, consider a continuously differentiable function $h: \mathbb{R}^n \mapsto {\mathbb{R}}$, which implicitly defines a (closed) safe set of system states $\mathcal{S}  \delequal \{\bfx \in \mathcal{X} \; | \; h(\bfx) \geq 0\}$. The following definition is a useful tool to ensure that $\mathcal{S}$ is forward invariant, i.e., the robot state remains in $\mathcal{S}$ throughout its evolution.

%
%


\begin{definition}
A continuously differentiable function $h: \mathbb{R}^n \mapsto {\mathbb{R}}$ is a \emph{control barrier function (CBF)} on $\mathcal{X} \subseteq \mathbb{R}^n$ for \eqref{eq: dynamic} if there exists an extended class $\mathcal{K}_{\infty}$ function $\alpha_h$ with:
\begin{equation}\label{eq:cbf}
    \sup_{\ubfu\in \underline{\mathcal{U}}} \textit{CBC}(\bfx,\ubfu) \geq 0, \quad \forall \bfx \in \calX,
\end{equation}
where the \emph{control barrier condition (CBC)} is:
\begin{equation}
\label{eq:cbc_define}
\begin{aligned}
    \textit{CBC}(\bfx,\ubfu) & \delequal \mathcal{L}_f h(\bfx) + \mathcal{L}_g h(\bfx)\bfu + \alpha_h (h(\bfx)) \\
    & = [\nabla_{\bfx} h(\bfx)]^\top F(\bfx)\ubfu + \alpha_{h}(h(\bfx)).
\end{aligned}
\end{equation}
\end{definition}

According to~\cite{cbf,ames2016control}, any Lipschitz-continuous controller $\underline{\bfk}: \mathcal{X} \mapsto \underline{\mathcal{U}}$ that satisfies $\textit{CBC}(\bfx,\underline{\bfk}(\bfx)) \geq 0$ for all $\bfx \in \mathcal{X}$ renders the set $\mathcal{S}$ forward invariant for the system~\eqref{eq: dynamic}. 

\subsection{Safety and Stability with Known System Dynamics and Barrier Function}

When the system dynamics $F(\bfx)$ and barrier function $h(\bfx)$ are known, a safe controller can be synthesized by combining CLF and CBF constraints in a quadratic program:
\begin{equation}\label{eq:QP_origin}
\begin{aligned}
& \min_{\ubfu \in \underline{\calU},\delta \in \bbR}\,\, \|L(\bfx)^\top(\ubfu - \underline{\tilde{\bfk}}(\bfx))\|^2 + \lambda \delta^2,\\
\mathrm{s.t.} \, \,  &\textit{CLC}(\bfx,\ubfu) \leq \delta,\; \textit{CBC}(\bfx,\ubfu) \geq 0.
\end{aligned}
\end{equation}
%
The term $\underline{\tilde{\bfk}}(\bfx)$ is a baseline controller and may be used to specify additional control requirements, such as desirable velocity or orientation. This term may be set to $\underline{\tilde{\bfk}}(\bfx) \equiv \boldsymbol{e}_1$ if minimum control effort is the main objective. The term $L(\bfx)$ is a weighting matrix penalizing deviation from the baseline controller. The term $\delta \geq 0$
is a slack variable that relaxes the CLF constraints to ensure the feasibility of the QP, controlled by the scaling factor $\lambda > 0$. The QP formulation in \eqref{eq:QP_origin} modifies the baseline controller $\underline{\tilde{\bfk}}(\bfx)$ online to ensure safety and stability via the CBF and CLF constraints.

We focus on enforcing safety and stability for the control-affine system in \eqref{eq: dynamic} when the system dynamics $F(\bfx)$ and the barrier function $h(\bfx)$ are \emph{unknown} and need to be estimated from data.
%
%
We present an approach for estimating the system dynamics and barrier functions from data in Sec.\ref{sec:dynamics_estimate} and Sec.\ref{sec:sdf_estimates}, respectively. Our main goal is to develop techniques for safe and stable control synthesis with the estimated $F(\bfx)$ and $h(\bfx)$. We consider two scenarios, depending on whether probabilistic or worst-case error descriptions of the dynamics and barrier functions are available. 

\subsection{Safety and Stability with Gaussian Process Distributed System Dynamics and Barrier Function}

When the system dynamics and barrier functions can be described as GPs, we consider the following probabilistic control synthesis problem.

%
%



\begin{problem}[\textbf{Safety and stability under Gaussian uncertainty}]\label{prob: prob_clf_cbf}
Given an estimated distribution on the unknown system dynamics $\text{vec}(F(\bfx)) \sim \mathcal{GP}(\text{vec}(\tilde{F}(\bfx)), K_{F}(\bfx,\bfx'))$ and an estimated distribution on the barrier function $h(\bfx) \sim \mathcal{GP}(\tilde{h}(\bfx),K_h(\bfx,\bfx'))$, design a feedback controller $\underline{\bfk}$ such that, for each $\bfx \in \calX$:
\begin{equation*}
    \mathbb{P}(\textit{CLC}(\bfx, \underline{\bfk}(\bfx)) \leq \delta) \geq p, \quad \mathbb{P}(\textit{CBC}(\bfx, \underline{\bfk}(\bfx)) \geq 0) \geq p,
\end{equation*}
where $p \in (0,1)$ is a user-specified risk tolerance. 
\end{problem}

\subsection{Safety and Stability with Worst-Case Uncertainty in System Dynamics and Barrier Function}

Many robotic systems require instead the guarantee that safety and stability hold under all possible error realizations,
which motivates us to also consider the following problem.

\begin{problem}[\textbf{Safety and stability under worst-case uncertainty}]\label{prob: robust_clf_cbf}
Given estimated system dynamics $\tilde{F}(\bfx)$ with known error bound $e_F(\bfx)$,
\begin{equation}
\label{error_dynamics}
    \|F(\bfx) - \tilde{F}(\bfx)\| \leq e_F(\bfx) , \; \forall \bfx \in \calX ,
\end{equation}
and estimated barrier function $\tilde{h}(\bfx)$ and gradient $\nabla \tilde{h}(\bfx)$ with known error bounds $e_h(\bfx)$ and $e_{\nabla h}(\bfx)$, i.e., for all $\bfx \in \calX$,
\begin{equation}
\label{error_barrier}
     |h(\bfx) -\tilde{h}(\bfx)| \leq e_h(\bfx) , \;
    \|\nabla h(\bfx) - \nabla \tilde{h}(\bfx)\|  \leq e_{\nabla h}(\bfx),
\end{equation}
%
design a feedback controller $\underline{\bfk}$ such that, for each $\bfx \in \calX$:
\begin{equation*}
    \textit{CLC}(\bfx, \underline{\bfk}(\bfx)) \leq \delta, \quad \textit{CBC}(\bfx, \underline{\bfk}(\bfx)) \geq 0.
\end{equation*}
\end{problem}
%








\section{Probabilistic Safe Control}\label{sec:prob_control}

This section presents our solution to Problem~\ref{prob: prob_clf_cbf}. 
Inspired by the design~\eqref{eq:QP_origin} when the dynamics and the barrier function are known, we formulate the control synthesis problem via the following optimization problem:
\begin{align}\label{eq:CLF_CBF_QP_Gaussian}
& \min_{\ubfu \in \underline{\calU},\delta \in \bbR}\,\, \|L(\bfx)^\top(\ubfu - \underline{\tilde{\bfk}}(\bfx))\|^2 + \lambda \delta^2,\\
&\mathrm{s.t.} \, \,  
\mathbb{P}(\textit{CLC}(\bfx, \ubfu) \leq \delta) \geq p, \quad \mathbb{P}(\textit{CBC}(\bfx, \ubfu) \geq 0) \geq p. \notag
\end{align}
The uncertainty in $F$ and $h$ affects the linearity in $\ubfu$ of the CLC and CBC conditions in the constraints of~\eqref{eq:CLF_CBF_QP_Gaussian}, making this optimization problem no longer a QP. Here, we justify that nevertheless the optimization can be solved efficiently. To show this, we start by analyzing the distributions of $\textit{CBC}(\bfx,\ubfu)$ and $\textit{CLC}(\bfx,\ubfu)$ in detail. 

\begin{proposition}[\textbf{Mean and Variance for CBC}]\label{prop:cbc_mean_variance}
Assume $h$ is a CBF with a linear function $\alpha_h$, i.e., $\alpha_h(z) = a \cdot z$ for $a \in \mathbb{R}_{\ge 0}$. Given independent distributions $h(\bfx) \sim \mathcal{GP}(\tilde{h}(\bfx),K_h(\bfx,\bfx'))$ and $\text{vec}(F(\bfx)) \sim \mathcal{GP}(\text{vec}(\tilde{F}(\bfx)), K_{F}(\bfx,\bfx'))$, the mean and variance of $\textit{CBC}(\bfx,\ubfu)$ satisfy
\begin{subequations}\label{eq: cbc_mean_variance}
\begin{align}
  \label{eq:cbc_mean}
    &\mathbb{E}[\textit{CBC}(\bfx,\ubfu)] = \mathbb{E}[\bfp(\bfx)]^\top\ubfu
    \\
    &\textit{Var}[\textit{CBC}(\bfx,\ubfu)] = \ubfu^\top \textit{Var}[\bfp(\bfx)]\ubfu,     \label{eq:cbc_var}
\end{align}
\end{subequations}
where $\bfp(\bfx) := F^\top(\bfx)[\nabla_{\bfx} h(\bfx)]+  \left[a h(\bfx) \quad \boldsymbol{0}_m^\top \right]^\top \in \bbR^{m+1}$
and $\mathbb{E}[\bfp(\bfx)]$, $ \textit{Var}[\bfp(\bfx)]$ are computed in \eqref{eq:covariance_Px}.
\end{proposition}
\begin{proof}
The control barrier condition can be written as:
\begin{align*}
    \textit{CBC}(\bfx,\ubfu) &= [\nabla_{\bfx} h(\bfx)]^\top f(\bfx) + [\nabla_{\bfx} h(\bfx)]^\top g(\bfx)\bfu + a h(\bfx)  
    \notag
    \\
     &= \big[ [\nabla_{\bfx} h(\bfx)]^\top F(\bfx) \!+\! \left[a h(\bfx) \; \boldsymbol{0}_m^\top \right] \big ] \ubfu = \bfp(\bfx)^\top \ubfu .
\end{align*}
%
Note that $\nabla_{\bfx} h(\bfx)$ is a GP because the gradient of a GP with differentiable mean function and twice-differentiable covariance function is also a GP, 
cf.~\cite[Lemma 6]{dhiman2020control},
\begin{equation*}
\begin{aligned}
    \nabla_{\bfx} h(\bfx) \sim
    \mathcal{GP}(\nabla_{\bfx} \tilde{h}(\bfx),\mathcal{H}_{\bfx,\bfx'}K_h(\bfx,\bfx')),
\end{aligned}
\end{equation*}
where $\mathcal{H}_{\bfx,\bfx'}K_h(\bfx,\bfx') = \left[ \frac{\partial^2 K_h(\bfx,\bfx')}{\partial\bfx_i,\partial\bfx_j'} \right]_{i=1,j=1}^{n,n}$ is finite for all $(\bfx,\bfx') \in \mathbb{R}^{2n}$. Since $\text{vec}(\bfA\bfB\bfC) = (\bfC^\top \otimes \bfA)\text{vec}(\bfB)$ for appropriately sized matrices $\bfA$, $\bfB$, $\bfC$, we can write 
\begin{equation}
\label{eq:kronecker_trick}
\begin{aligned}
\Var(F(\bfx) \ubfu)
&= \Var((\ubfu^\top \otimes \bfI_{n})\text{vec}(F(\bfx)))\\
&= (\ubfu^\top \otimes \bfI_{n})K_F(\bfx, \bfx) (\ubfu \otimes \bfI_{n}) .
\end{aligned}
\end{equation}
For brevity, we let $K_F := K_F(\bfx,\bfx')$ and $K_h := K_h(\bfx, \bfx')$ and $\bfp_1 = F^\top(\bfx)[\nabla_\bfx h(\bfx)]$. The term $[\nabla_{\bfx} h(\bfx)]^\top F(\bfx)\ubfu$ is an inner product of two independent GPs, $\nabla_{\bfx} h(\bfx)$ and $F(\bfx)\ubfu$.
Thus, using \cite[Lemma 5]{dhiman2020control}, \eqref{eq:kronecker_trick}, and that $\Cov(\nabla_\bfx h(\bfx), F(\bfx)\ubfu) = 0$, $\bfp_1^\top\ubfu$ corresponds to a distribution with mean and variance:
\begin{equation}
\label{eq:part_1_mean_variance}
\begin{aligned}
&\bbE[\bfp_1^\top \ubfu] = [\nabla_{\bfx} \tilde{h}(\bfx)]^\top\tilde{F}(\bfx)\ubfu , \\
&\textit{Var}[\bfp_1^\top\ubfu] = 
 [\nabla_{\bfx} \tilde{h}(\bfx)]^\top(\ubfu^\top \otimes \bfI_n)K_{F} 
  \\&\qquad
(\ubfu \otimes \bfI_n) \nabla_{\bfx} \tilde{h}(\bfx) + 
 \ubfu^\top\tilde{F}^\top(\bfx)\mathcal{H}_{\bfx,\bfx'} K_h \tilde{F}(\bfx) \ubfu.
\end{aligned}
\end{equation}
To factorize $\ubfu$ from the variance expression, we apply
the property $(\bfA \otimes \bfB)(\bfC \otimes \bfD) = \bfA\bfC \otimes \bfB\bfD$ two times,
\begin{equation}
\label{eq:swapping-gradh-u}
\begin{aligned}
    &(\ubfu \otimes \bfI_n)[\nabla_\bfx \tilde{h}(\bfx)] =
     (\ubfu \otimes \bfI_n)(1 \otimes [\nabla_\bfx \tilde{h}(\bfx)])
     \\&\quad= \ubfu \otimes \nabla_\bfx \tilde{h}(\bfx) 
    = (\bfI_{m+1} \otimes \nabla_\bfx \tilde{h}(\bfx)) \ubfu.
\end{aligned}
\end{equation}
By substituting \eqref{eq:swapping-gradh-u} in \eqref{eq:part_1_mean_variance}, we can factorize out $\ubfu$ to get,
\begin{align}
\label{eq:p1-var}
\Var[\bfp_1]
 &=  
 (\bfI_{m+1} \otimes [\nabla_\bfx \tilde{h}(\bfx)]^\top)K_{F}
 (\bfI_{m+1} \otimes \nabla_\bfx \tilde{h}(\bfx)) 
 \notag\\ &\quad 
 +  \tilde{F}^\top(\bfx)\mathcal{H}_{\bfx,\bfx'} K_h \tilde{F}(\bfx)  . 
\end{align}
Next, we write $\Cov(h(\bfx), \bfp_1^\top\ubfu)$ using \cite[Lemma 5]{dhiman2020control} and $\Cov(h(\bfx), F(\bfx)\ubfu) = 0$,
\begin{align}
\label{eq:p1-h-cov}
\Cov(&h(\bfx), \bfp_1^\top\ubfu) =
\Cov(h(\bfx),  \nabla_\bfx h(\bfx))\tilde{F}(\bfx)\ubfu
\notag \\ &\quad
= \big[[\nabla_\bfx K_h]^\top\tilde{f}(\bfx) \quad [\nabla_\bfx K_h]^\top\tilde{g}(\bfx)\big]\ubfu.
\end{align}
Using \eqref{eq:part_1_mean_variance}, \eqref{eq:p1-var} and \eqref{eq:p1-h-cov}, we write the mean and variance,
\begin{align}
\label{eq:covariance_Px}
&\bbE[\bfp(\bfx)] = [\nabla_\bfx \tilde{h}(\bfx)]^\top \tilde{F}(\bfx) + a [\tilde{h}(\bfx) \quad \boldsymbol{0}_m^\top]^\top
 \notag\\
&\Var[\bfp(\bfx)] =  
\tilde{F}^\top(\bfx)\mathcal{H}_{\bfx,\bfx'} K_h \tilde{F}(\bfx)  
\notag\\ &\;
+
(\bfI_{m+1} \otimes \nabla_\bfx \tilde{h}(\bfx)^\top)K_{F}
 (\bfI_{m+1} \otimes \nabla_\bfx \tilde{h}(\bfx)) 
\\ &\;
+
\begin{bmatrix}
a^2 K_h + 2a[\nabla_\bfx K_h]^\top \tilde{f}(\bfx)& a[\nabla_\bfx K_h]^\top \tilde{g}(\bfx)\\
a\tilde{g}(\bfx)^\top [\nabla_\bfx K_h] &  \boldsymbol{0}_{m \times m}
\end{bmatrix} , \notag
\end{align}
from which the statement follows.
\end{proof}

Next, we describe the distribution of $\textit{CLC}(\bfx,\ubfu)$.

\begin{proposition}[\textbf{Gaussian distribution for CLC}]\label{prop:clc_mean_variance}
Given the distribution $\text{vec}(F(\bfx)) \sim \mathcal{GP}(\text{vec}(\tilde{F}(\bfx)), K_{F}(\bfx,\bfx'))$, the $\textit{CLC}(\bfx,\ubfu)$ is Gaussian with mean and variance:
\begin{subequations}\label{eq: clc_mean_variance}
\begin{align}
\mathbb{E}[\textit{CLC}(\bfx,\ubfu)] &\!=\! 
\mathbb{E}[\bfq(\bfx)]^\top  \ubfu   \\
\textit{Var}[\textit{CLC}(\bfx,\ubfu)] & \!=\!  \ubfu^\top \textit{Var}[\bfq(\bfx)] \ubfu,
\end{align}
\end{subequations}
where $
    \bfq(\bfx) := F^\top(\bfx) [\nabla_{\bfx} V(\bfx)] + [\alpha_V(V(\bfx)) \quad \boldsymbol{0}_m^\top ]^\top\in \mathbb{R}^{m+1}$ and $\mathbb{E}[\bfq(\bfx)]$, $ \textit{Var}[\bfq(\bfx)]$ are computed in \eqref{eq:Qx_mean_variance}.
\end{proposition}
\begin{proof}
We can write the control Lyapunov condition as
\label{eq:gaussian_clf}
$    \textit{CLC}(\bfx,\ubfu)  =  [\nabla_{\bfx}  V(\bfx)]^\top F(\bfx)\ubfu + \alpha_{V}(V(\bfx)) = \bfq^\top(\bfx) \ubfu$.
%
We use the Kronecker product property  $\text{vec}(\bfA\bfB\bfC) = (\bfC^\top \otimes \bfA)\text{vec}(\bfB)$ to rewrite first term in $\bfq(\bfx)$ as:
\begin{equation*}
\begin{aligned}
[\nabla_{\bfx} V(\bfx)]^\top F(\bfx) 
&= (\bfI_{m+1} \otimes [\nabla_{\bfx} V(\bfx)]^\top)\text{vec}(F(\bfx)).
\end{aligned}    
\end{equation*}
Since $[\nabla_{\bfx} V(\bfx)]$, $\alpha_{V}(V(\bfx))$ are known and deterministic and $\text{vec}(F(\bfx)) \sim \mathcal{GP}(\text{vec}(\tilde{F}(\bfx)), K_{F}(\bfx,\bfx'))$, we can express the distribution of $\bfq(\bfx)$ as follows:
\begin{align}
\label{eq:Qx_mean_variance}
\mathbb{E}[\bfq(\bfx)] &\!=\!  \tilde{F}^\top(\bfx) [\nabla_{\bfx} V(\bfx)] +[\alpha_V(V(\bfx)) \quad \boldsymbol{0}_{m}^\top ]^\top \\
\textit{Var}[\bfq(\bfx)]  &\!=\! (\bfI_{m+1} \otimes [\nabla_{\bfx} V(\bfx)]^\top)K_{F}(\bfI_{m+1} \otimes [\nabla_{\bfx} V(\bfx)]) \nonumber.
\end{align}
The result follows from plugging \eqref{eq:Qx_mean_variance} into $\textit{CLC}(\bfx,\ubfu)$.
\end{proof}

We use the mean and variance of $\textit{CBC}(\bfx,\ubfu)$ and $\textit{CLC}(\bfx,\ubfu)$ obtained above to approximate the probabilistic safety and stability constraints in \eqref{eq:CLF_CBF_QP_Gaussian}.


\begin{proposition}[\textbf{Probabilistic CLF-CBF SOCP}]\label{prop:gaussian_socp}
Given a user-specified risk tolerance $p \in [0,1)$, let $c(p) = \sqrt{\frac{p}{1-p}}$. The optimization problem \eqref{eq:CLF_CBF_QP_Gaussian} can be formulated as the following second-order cone program:
\begin{equation}
\label{eq:gaussian_socp_form}
\begin{aligned}
    &\min_{\ubfu \in \underline{\calU}, \delta\in \mathbb{R},l\in \mathbb{R} } \, \, l& \qquad \qquad \\
    &\mathrm{s.t.} \, \, 
    \delta - \mathbb{E}[\bfq(\bfx)]^\top\ubfu  \geq c(p)\sqrt{\ubfu^\top\textit{Var}[\bfq(\bfx)]\ubfu},
    \\
    &\qquad \quad \mathbb{E}[\bfp(\bfx)]^\top\ubfu \geq c(p)\sqrt{\ubfu^\top\textit{Var}[\bfp(\bfx)]\ubfu},
     \\
    &l+1 \geq
    \sqrt{ \|2L(\bfx)^\top(\ubfu-\underline{\tilde{\bfk}}(\bfx))\|^2 + (2\sqrt{\lambda}\delta)^2 + (l-1)^2}
\end{aligned}
\end{equation}
where $\bfp$, $\bfq$ are defined in
Propositions~\ref{prop:cbc_mean_variance} and~\ref{prop:clc_mean_variance}, resp.
\end{proposition}
\begin{proof}
To deal with the probabilistic constraints in~\eqref{eq:CLF_CBF_QP_Gaussian}, we employ Cantelli's inequality \cite{CantelliSuiCD}. For any scalar $\gamma \geq 0$,
\begin{equation*}
\begin{aligned}
\label{eq:cantelli_1}
    \mathbb{P}(\textit{CBC}(\bfx,\ubfu) \geq \mathbb{E}[\textit{CBC}(\bfx,\ubfu))] -\gamma | \bfx, \ubfu) \geq \\
    1 - \frac{\textit{Var}[\textit{CBC}(\bfx,\ubfu)]}{\textit{Var}[\textit{CBC}(\bfx,\ubfu)]+\gamma^2}  .
\end{aligned}
\end{equation*}
Given this inequality, and since we want $\mathbb{P}(\textit{CBC}(\bfx,\ubfu) \ge 0) \geq p$, we choose $\gamma =  \mathbb{E}[\textit{CBC}(\bfx,\ubfu)]$
%
%
and require the lower bound to be greater than or equal to $p$, i.e.,
$
1 - \frac{\textit{Var}[\textit{CBC}(\bfx,\ubfu)]}{\textit{Var}[\textit{CBC}(\bfx,\ubfu)]+\gamma^2} \geq p$. 
The equation can be rearranged into
\begin{equation*}
\label{eq:cantelli_3}
    \mathbb{E}[\textit{CBC}(\bfx,\ubfu)] = \gamma \geq \sqrt{\frac{p}{1-p}\textit{Var}[\textit{CBC}(\bfx,\ubfu)]},
\end{equation*}
which corresponds to the safety constraint in \eqref{eq:gaussian_socp_form}.

Next, we show that this is a second-order cone (SOC) constraint. By \eqref{eq: cbc_mean_variance}, given that $\tilde{h}$, $\nabla\tilde{h}$ and $\tilde{F}$ are known and deterministic, the expectation $\mathbb{E}[\textit{CBC}(\bfx,\ubfu)] = \mathbb{E}[\bfp(\bfx)]^\top\ubfu $ is affine in $\ubfu$. Since $\textit{Var}[\bfp(\bfx)]$ is positive semi-definite,
\begin{equation}
\label{eq:cbc_var_norm}
    \sqrt{\textit{Var}[\textit{CBC}(\bfx,\ubfu)]} = \sqrt{\ubfu^\top\textit{Var}[\bfp(\bfx)]\ubfu} = \|\bfD(\bfx)\ubfu\|
\end{equation}
where $\bfD(\bfx)^\top \bfD(\bfx) = \textit{Var}[\bfp(\bfx)]$. Acccording to~\cite{alizadeh2003second}, the safety constraint in \eqref{eq:gaussian_socp_form} is a valid SOC constraint.

For stability, the CLC condition can be constructed using a similar approach with Cantelli's inequality, resulting in~\eqref{eq:gaussian_socp_form}. By \eqref{eq: clc_mean_variance}, we know that the expectation is affine in $\ubfu$ and the variance is quadratic in terms of $\ubfu$, similar to \eqref{eq:cbc_var_norm}. This shows that the CLC condition is also a valid SOC constraint.

Our last step is to reformulate the minimization of the objective function as a linear objective with an SOC constraint, resulting in the standard SOCP in~\eqref{eq:gaussian_socp_form}. We introduce a new variable $l$ so that the problem in \eqref{eq:CLF_CBF_QP_Gaussian} is equivalent to
\begin{align}
\label{eq:CLF_CBF_QP_Gaussian_Convert}
    & \min_{\ubfu \in \underline{\calU}, \delta\in \mathbb{R},l\in \mathbb{R} } \, \, l \notag \\
    \text{s.t.} \, \, 
    &\mathbb{P}(\textit{CLC}(\bfx, \ubfu) \leq \delta) \geq p, \quad \mathbb{P}(\textit{CBC}(\bfx, \ubfu) \geq 0) \geq p, \notag \\
    &\|L(\bfx)^\top(\ubfu - \underline{\tilde{\bfk}}(\bfx))\|^2 + \lambda \delta^2 \leq l.
\end{align}
The last constraint in \eqref{eq:CLF_CBF_QP_Gaussian_Convert} corresponds to a rotated second-order cone, $\mathcal{Q}^n_{rot} \coloneqq \{(\bfx_{r},y_{r},z_{r}) \in \mathbb{R}^{n+2}\, | \,\|\bfx_{r}\|^2 \leq y_{r}z_{r}, y_{r} \geq 0, z_{r} \geq 0 \}$, which can be converted into a standard SOC constraint~\cite{alizadeh2003second},
$
\left\| \begin{bmatrix} 2 \bfx_{r} & y_r - z_r\end{bmatrix}^\top \right\| \leq y_r+z_r.
$
Let $y_r = l$, $z_r = 1$ and consider the constraint $\|L(\bfx)^\top(\ubfu-\underline{\tilde{\bfk}}(\bfx))\|^2 + \lambda \delta^2 \leq l$. Multiplying both sides by $4$ and adding $(l-1)^2$, makes the constraint equivalent to 
\[
4\|L(\bfx)^\top(\ubfu-\underline{\tilde{\bfk}}(\bfx))\|^2 + 4\lambda\delta^2 + (l-1)^2 \leq (l+1)^2.
\]
Taking a square root on both sides, we end up with $\sqrt{ \|2L(\bfx)^\top(\ubfu-\underline{\tilde{\bfk}}(\bfx))\|^2 + (2\sqrt{\lambda}\delta)^2 + (l-1)^2} \leq l+1$, which is equivalent to the third constraint in~\eqref{eq:gaussian_socp_form}.
\end{proof}

\begin{remark}[\textbf{Effects of risk-tolerance $p$ and variance}]
When $p=0$, the probabilistic CLF-CBF-SOCP \eqref{eq:gaussian_socp_form} reduces to the original CLF-CBF-QP \eqref{eq:QP_origin}. As $p$ and/or $\textit{Var}[\bfp(\bfx)]$, $\textit{Var}[\bfq(\bfx)]$ increase, the feasible region of \eqref{eq:gaussian_socp_form} gets smaller, and the optimal value worsens, cf. Fig.~\ref{fig:2b} for an illustration. 
\end{remark}

\section{Robust Safe Control}
\label{sec:robust_control}

In this section, we develop a solution to Problem~\ref{prob: robust_clf_cbf}. 
Let $\tilde{F}$ denote the estimated system dynamics, $\tilde{h}$, $\nabla \tilde{h}$ the estimated barrier function and its gradient, and let $e_F : \mathbb{R}^{n \times (m+1)} \mapsto \mathbb{R}_{\geq 0}$, $e_h : \mathbb{R} \mapsto \mathbb{R}_{\geq 0}$, and $e_{\nabla h}: \mathbb{R}^n \mapsto \mathbb{R}_{\geq 0}$ be associated error bounds. For convenience, for each $\bfx \in \mathcal{X}$, we denote $D_F(\bfx) := F(\bfx) - \tilde{F}(\bfx)$, $d_h(\bfx) := h(\bfx) -\tilde{h}(\bfx)$ and $\bfd_{\nabla h}(\bfx) := \nabla h(\bfx) - \nabla \tilde{h}(\bfx)$. 
By \eqref{error_dynamics} and~\eqref{error_barrier}, we have
\begin{equation}
\label{eq: error_bound}
    \|D_F (\bfx)\|  \leq e_F(\bfx), \: |d_h (\bfx)|  \leq e_h(\bfx), \:  \|\bfd_{\nabla h} (\bfx) \|   \leq e_{\nabla h}(\bfx).
\end{equation}
Using this notation, we can rewrite $\textit{CBC}(\bfx,\ubfu)$ as
\begin{equation*}
\begin{aligned}
&\textit{CBC}(\bfx, \ubfu) = [\nabla h(\bfx)]^\top  F(\bfx)\ubfu  +  \alpha_h( h(\bfx)) \\
&= [\nabla \tilde{h}(\bfx)]^\top \tilde{F}(\bfx)\ubfu +  \bfd_{\nabla h}^\top(\bfx)\tilde{F}(\bfx)\ubfu+ [\nabla \tilde{h}(\bfx)]^\top D_F(\bfx)\ubfu \\ 
&\quad + \bfd_{\nabla h}^\top(\bfx) D_F(\bfx)\ubfu +  \alpha_h( \tilde{h}(\bfx)+d_h(\bfx)).
\end{aligned}
\end{equation*}
Let $\tilde{\bfp}(\bfx) := \tilde{F}^\top(\bfx)\nabla \tilde{h}(\bfx) $. We group the error term in the expression for $\textit{CBC}(\bfx, \ubfu)$ in the variable $d_{\textit{CBC}}(\bfx, \ubfu) := \textit{CBC}(\bfx, \ubfu)- \tilde{\bfp}(\bfx)^\top \ubfu$.
Thus, $\textit{CBC}(\bfx,\ubfu) \geq 0$ is satisfied if
%
\begin{align*}
 \min_{D_F,\bfd_{\nabla h}, d_h}\! \textit{CBC}(\bfx, \ubfu) = \tilde{\bfp}(\bfx)^\top \ubfu  + \! \min_{D_F,\bfd_{\nabla h}, d_h} \! d_{\textit{CBC}}(\bfx, \ubfu)  \ge 0.
\end{align*}
%
Similarly, let 
$\tilde{\bfq}(\bfx) := \tilde{F}^\top(\bfx)\nabla V(\bfx)   + [\alpha_{V}(V(\bfx))  \quad \boldsymbol{0}_{m}^\top]^\top$ and $d_{\textit{CLC}}(\bfx, \ubfu) := [\nabla V(\bfx)]^\top D_F(\bfx)\ubfu$, a robust version of the stability constraint $\textit{CLC}(\bfx,\ubfu) \leq \delta$ can be written as:
%
%
\begin{equation}
\label{eq:robust_clc}
\max_{D_F} \textit{CLC}(\bfx, \ubfu) = \tilde{\bfq}(\bfx)^\top \ubfu  + \max_{D_F}  d_{\textit{CLC}}(\bfx, \ubfu) \leq \delta .
\end{equation}
This leads us to the following robust reformulation of the original control synthesis problem in~\eqref{eq:QP_origin},
\begin{equation}\label{eq:QP_robust}
\begin{aligned}
    &\min_{\ubfu \in \underline{\calU}, \delta\in \mathbb{R},l\in \mathbb{R} } \, \, l \qquad \qquad \\
    \mathrm{s.t.} \, \,   
    &\tilde{\bfq}(\bfx)^\top \ubfu + \max_{D_F} \; d_{\textit{CLC}}(\bfx, \ubfu) \le \delta \\
    & \tilde{\bfp}(\bfx)^\top \ubfu + \min_{D_F,  d_h, \bfd_{\nabla h}} d_{\textit{CBC}}(\bfx, \ubfu) \ge 0
    \\
    &l+1 \geq
    \sqrt{ \|2L(\bfx)^\top(\ubfu-\underline{\tilde{\bfk}}(\bfx))\|^2 + (2\sqrt{\lambda}\delta)^2 + (l-1)^2}.
\end{aligned}
\end{equation}
Note that we used the same approach as in the proof of Proposition~\ref{prop:gaussian_socp} to reformulate the original quadratic objective with a linear objective plus a SOC constraint. The second constraint in \eqref{eq:QP_robust} requires solving $\min_{D_F,  d_h, \bfd_{\nabla h}} d_{\textit{CBC}}(\bfx, \ubfu)$ subject to \eqref{eq: error_bound}. In general, this is a non-convex 
constrained quadratic program which does not have a closed-form expression of the minimizer as a function of~$\ubfu$. Instead, we make the second constraint in \eqref{eq:QP_robust} more conservative using the Cauchy-Schwarz inequality, which leads to a convex SOCP, whose optimal solution is guaranteed to be feasible for \eqref{eq:QP_robust}.

\begin{proposition}[\textbf{Robust CLF-CBF SOCP}]\label{prop: worst_case_error}
Let $\tilde{F}$, $\tilde{h}$, $\nabla \tilde{h}$ denote estimates of the system dynamics and barrier function, with error bounds in \eqref{eq: error_bound}. Then, the feasible set of the following SOCP is included in the feasible set of \eqref{eq:QP_robust}: 
%
\begin{align}
\label{eq:worst_case_SOCP_formulation}
    & \min_{\ubfu \in \underline{\calU}, \delta\in \mathbb{R}, p\in \mathbb{R}, q \in \mathbb{R}, l\in \mathbb{R}} \, \, l \notag \\
    \mathrm{s.t.} \, \, 
    &\delta - \tilde{\bfq}(\bfx)^\top \ubfu \geq e_F(\bfx)\|\nabla V(\bfx)\| \|\ubfu\|,  \notag \\
    & p \geq e_{\nabla h}(\bfx)\|\tilde{F}(\bfx) \ubfu \|  , \notag \\
    & q \geq \Big(e_F(\bfx)\|\nabla \tilde{h}(\bfx)\|  + e_{\nabla h}(\bfx)e_F(\bfx)\Big)\|\ubfu \| , \notag \\
    & [\nabla \tilde{h}(\bfx)]^\top \tilde{F}(\bfx)\ubfu + \alpha_h(\tilde{h}(\bfx) - e_h(\bfx)) \geq p + q , \notag \\
    &l+1 \geq
    \sqrt{ \|2L(\bfx)^\top(\ubfu-\underline{\tilde{\bfk}}(\bfx))\|^2 + (2\sqrt{\lambda}\delta)^2 + (l-1)^2} 
\end{align}
\end{proposition}

\begin{proof}
%
%
%
%
The stability constraint in \eqref{eq:QP_robust} is reformulated using:
\begin{equation*}
    \max_{\|D_F (\bfx)\|  \leq e_F(\bfx)}d_{\textit{CLC}}(\bfx, \ubfu) = e_F(\bfx)\|\nabla V(\bfx)\|\|\ubfu\|.
\end{equation*}
%
%
%
For the safety constraint in \eqref{eq:QP_robust}, note that
\begin{align}
&\min_{D_F,  d_h, \bfd_{\nabla h}} d_{\textit{CBC}}(\bfx, \ubfu) \notag
\\
& = \min_{D_F,  \bfd_{\nabla h}} \Big(\bfd_{\nabla h}^\top(\bfx)\tilde{F}(\bfx)\ubfu + [\nabla \tilde{h}(\bfx)]^\top D_F(\bfx)\ubfu + \notag\\ 
&\qquad \bfd_{\nabla h}^\top(\bfx) D_F(\bfx)\ubfu \Big) + \min_{d_h} \alpha_h(\tilde{h}(\bfx) + d_h(\bfx)).
\end{align}
Since $e_h(\bfx) \geq 0$ and $\alpha_h$ is an extended class $\mathcal{K}_\infty$ function,
\begin{equation}
\label{eq:robust_cbc_rhs}
    \min_{|d_h(\bfx)| \leq e_{h}(\bfx)} \hspace*{-2pt} \alpha_h(\tilde{h}(\bfx) + d_h(\bfx)) \!=\! \alpha_h(\tilde{h}(\bfx) - e_h(\bfx)) . 
\end{equation}
Applying the Cauchy-Schwarz inequality on each term,
\begin{align*}
& \min_{D_F,  d_h, \bfd_{\nabla h}} d_{\textit{CBC}}(\bfx, \ubfu)
\ge -\|\bfd_{\nabla h}\|\|\tilde{F}(\bfx)\ubfu\| 
\notag\\  &\qquad 
- \|\nabla \tilde{h}(\bfx)\|\|D_F(\bfx)\ubfu\| -  \|\bfd_{\nabla h}(\bfx)\|\| D_F(\bfx)\ubfu\| 
\notag\\  &\qquad 
+  \alpha_h(\tilde{h}(\bfx) - e_h(\bfx))
\notag\\
&\quad \ge - e_{\nabla h}(\bfx) \|\tilde{F}(\bfx)\ubfu\| - e_F(\bfx)\|\nabla \tilde{h}(\bfx)\| \|\ubfu\| -
\notag\\
&\qquad e_{\nabla h}(\bfx) e_F(\bfx)\|\ubfu\| + \alpha_h(\tilde{h}(\bfx) - e_h(\bfx)).
\end{align*}
In the last step, we minimized each term independently, so the lower bound is not tight. We write the safety constraint~as
\begin{align}
\label{eq:robust_cbc_constraint}
    &e_{\nabla h}(\bfx)\|\tilde{F}(\bfx) \ubfu \|
    + (e_F(\bfx)\|\nabla \tilde{h}(\bfx)\|  + e_{\nabla h}(\bfx)e_F(\bfx))\|\ubfu \| \notag \\ 
    & \leq [\nabla \tilde{h}(\bfx)]^\top \tilde{F}(\bfx)\ubfu + \alpha_h(\tilde{h}(\bfx) - e_h(\bfx)).
\end{align}    
%
Constraints of the form $\|\bfA\bfz - \boldsymbol{a}\| + \|\bfB\bfz - \boldsymbol{b}\| \leq \boldsymbol{c}^\top \bfz$ can be replaced by the set of constraints $\|\bfA\bfz - \boldsymbol{a}\| \leq p$, $\|\bfB\bfz - \boldsymbol{b}\| \leq q$, $p+q \leq \boldsymbol{c}^\top \bfz$ combined. Thus,~\eqref{eq:robust_cbc_constraint} is equivalent to the second, third, and fourth constraints in \eqref{eq:worst_case_SOCP_formulation} together.
\end{proof}

\begin{remark}[\textbf{Effects of error bounds}]
If there are no errors in either the dynamics or the barrier function ($e_F \equiv e_h \equiv e_{\nabla h} \equiv 0$), then the robust CLF-CBF SOCP \eqref{eq:worst_case_SOCP_formulation} reduces to a CLF-CBF QP \eqref{eq:QP_origin}. If $e_F \equiv 0$ while $e_h(\bfx), e_{\nabla h}(\bfx) > 0$, the result in Proposition~\ref{prop: worst_case_error} recovers \cite[Proposition 2]{Long_learningcbf_ral21}. As the error bounds $e_F, e_h, e_{\nabla h}$ increase, the feasible region of~\eqref{eq:worst_case_SOCP_formulation} gets smaller and the optimal solution worsens. Also, note that the choice of kernel function, $K_F(\bfx,\bfx) = \frac{e^2_F(\bfx)}{c^2(p)}\bfI_{(m+1)n}$, reduces the inequality for stability in \eqref{eq:gaussian_socp_form} to that in \eqref{eq:worst_case_SOCP_formulation}. 
\end{remark}

\section{Evaluation}\label{evaluation}

In this section, we present an approach to estimate the unknown dynamics of a mobile robot, and construct CBF constraints online. Then, we evaluate our safe control synthesis using the estimated robot dynamics and CBFs in autonomous navigation tasks in $10$ simulated environments, containing obstacles a priori unknown to the robot.



%
%

%
\begin{figure}[t]
  \centering
  \subcaptionbox{Pybullet Simulator\label{fig:2a}}{\includegraphics[width=0.47\linewidth]{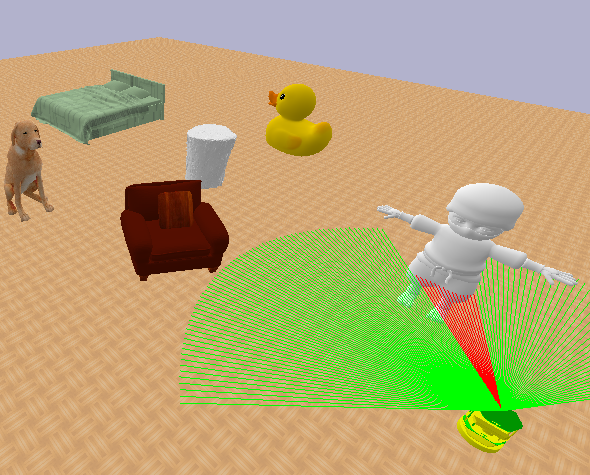}}%
  \hfill%
  \subcaptionbox{Probabilistic Trajectory\label{fig:2b}}{\includegraphics[width=0.48\linewidth]{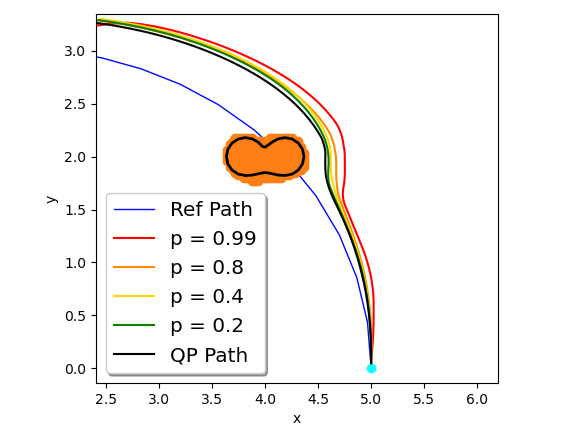}}\\
  \caption{(a) is the Pybullet simulation environment where we conduct our experiments. (b) shows the results in a region of an environment, where the probabilistic ($p=0.2, 0.4, 0.8, 0.99$) controller and QP controller both succeed. The ground-truth obstacle surface is shown in black while the estimated obstacles is shown in orange.}
  \vspace*{-3ex}
\end{figure}

\subsection{System Dynamics Estimation}\label{sec:dynamics_estimate}

We consider a Turtlebot robot simulated in the PyBullet simulator~\cite{coumans2020} (see Fig.~\ref{fig:2a}). We first present a learning approach to model the unknown dynamics of the TurtleBot using training data collected from the PyBullet simulator. The robot state and input are $\bfx := [x, y,\mu]^\top \in \mathbb{R}^2 \times [-\pi,\pi)$ and $\ubfu := [1, v,\omega]^\top \in \{1\} \times \mathbb{R}^2$, respectively. 
We collect a dataset $\mathcal{D} = \{t_{0:N}^{(i)}, \bfx_{0:N}^{(i)}, \ubfu_{0:N}^{(i)}\}_{i=1}^{D}$ of  $D=40000$ state sequences $\bfx_{0:N}^{(i)}$ obtained by applying random control inputs $\ubfu_{0:N}^{(i)}$ to the robot with initial condition $\bfx_{0}^{(i)}$ at time intervals of $\tau = 0.02$ seconds. For each trajectory $i$, a constant control input is applied for $N = 5$ time steps.
%
%



We employ a neural ODE network~\cite{chen2018neuralode} to approximate the unknown robot dynamics $F$ with a neural network $F_{\bftheta}$ based on the dataset $\mathcal{D}$. A forward pass through the ODE network is obtained using an ODE solver:
\begin{equation*}
\{\tilde{\bfx}_{1}^{i}, \tilde{\bfx}_{2}^{i}, \cdots, \tilde{\bfx}_{N}^{i}\} = \text{ODESolve}(\bfx_0^{i},F_{\bftheta}(\cdot)\ubfu^i, t_1^{i}, \cdots, t_N^{i}). 
\end{equation*}
We use a loss function,
\begin{equation}
\begin{aligned}
    &\min_{\boldsymbol{\theta}} \sum_{i=1}^{D} \sum_{j=1}^{N} \ell(\bfx_{j}^{(i)}, \tilde{\bfx}_{j}^{(i)}), \\
    \text{s.t.} \, \, &\dot{\tilde{\bfx}}^{(i)}(t) = F_{\bftheta}(\tilde{\bfx}^{(i)}(t)) \ubfu^{(i)}(t), \quad \tilde{\bfx}^{(i)}(j\tau) = \tilde{\bfx}^{(i)}_j,\\
    &\ubfu^{(i)}(t) \equiv \ubfu^{(i)}_j \;\;\text{for}\; t \in [j\tau,(j+1)\tau),
\end{aligned}
\end{equation}
where $\ell(\bfx, \tilde{\bfx}) = \|[x, y,\cos \mu,\sin \mu]^\top \!-\! [\tilde{x}, \tilde{y},\cos \tilde{\mu},\sin \tilde{\mu}]^\top\|^2$. To update the weights $\bftheta$, the gradient of the loss function is back-propagated by solving another ODE with adjoint states backwards in time. Please refer to \cite{chen2018neuralode} for details.

%
%
%

Gal and Ghahramani~\cite{2015yarin} showed that introducing dropout layers in a neural network is approximately equivalent to performing deep Gaussian Process regression.
We use a $6$-layer fully-connected neural network with $\tanh$ activations and $800$ neurons in each layer to model $F_{\bftheta}$, and apply dropout to each hidden layer with rate $0.05$. 
Given a query state $\bfx \in \calX$, Monte-Carlo estimates of the predictive mean $\tilde{F}_{\bftheta}(\bfx)$ and element-wise standard deviation $\tilde{\Sigma}(\bfx)$ of the dynamics are obtained with $T = 100$ stochastic forward passes through the dropout neural network model. We use $\tilde{F}_{\bftheta}(\bfx)$ for the mean of system dynamics and  $K_F(\bfx,\bfx) =  \diag(\text{vec}(\tilde{\Sigma}(\bfx))^2)$ for the variance of the dynamics. To obtain worst-case error bounds $e_F(\bfx)$, we set  $e_F(\bfx) = \|3.89\tilde{\Sigma}(\bfx)\|$ (99.99\% \textit{confidence}). 

In our experiment, no external disturbances are added to the system dynamics model. Given $M = 5000$ random-sampled different state control sequences $\{\bfx_i,\ubfu_i\}_{i=1}^{M}$ as test data, we consider the following test-time loss function, 
%
$L = \frac{1}{M} \sum_{i=1}^{M} \ell(F(\bfx_{i})\ubfu_i, \tilde{F}(\bfx_{i})\ubfu_i)$.
%
Our learned dynamics model is quite accurate, and the average test loss is $L = 0.0037$.

\subsection{Online CBF Estimation}
\label{sec:sdf_estimates}

\begin{table}[t]
\centering
\caption{Empirical SDF estimation error $\mathcal{E}$ and dropout-network SDF estimation error averaged across $8$ object instances under different LiDAR measurement noise standard deviation $\sigma$.}
\resizebox{\linewidth}{!}{
\begin{tabular}{ |c|c|c| }
\hline
LiDAR Noise $\sigma$ & SDF Empirical Error & SDF Dropout Error\\
\hline
$0.01$ & $0.0173$ & $0.0132$ \\
$0.02$ & $0.0288$ & $0.0184$ \\
$0.05$ & $0.0463$ & $0.0242$ \\
\hline 
\end{tabular}}
\label{table: sdf_error}
\vspace*{-1ex}
\end{table}

\begin{figure}[t]
\centering
\subcaptionbox{Training data \label{fig:1a}}{\includegraphics[width=0.25\linewidth]{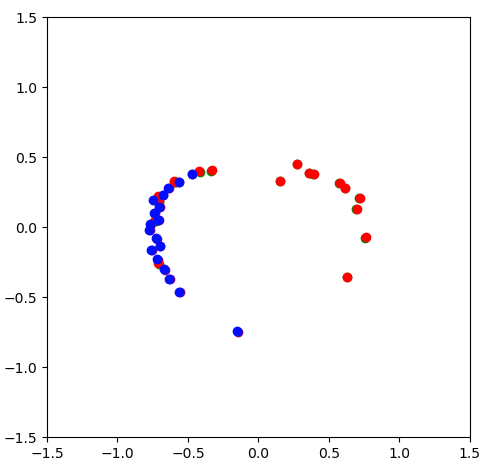}}%
\hfill%
\subcaptionbox{Mean SDF \label{fig:1b}}{\includegraphics[width=0.25\linewidth]{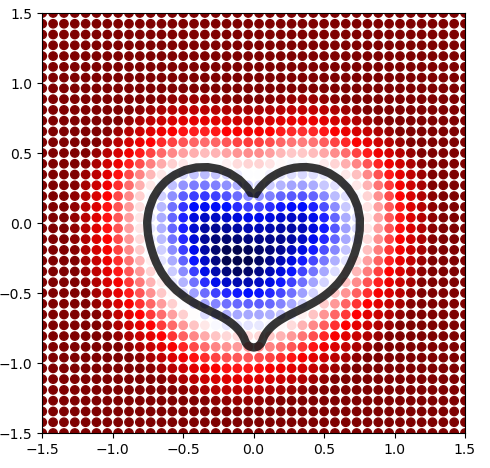}}%
\hfill%
\subcaptionbox{Variance \label{fig:1c}}{\includegraphics[width=0.245\linewidth]{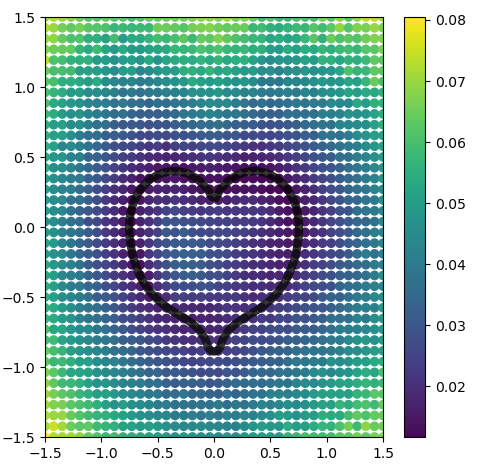}}%
\hfill%
\subcaptionbox{$\mathbb{P}(\tilde{\varphi} \!\leq\! 0) \!=\! 0.95\!$ \label{fig:1d}}{\includegraphics[width=0.25\linewidth]{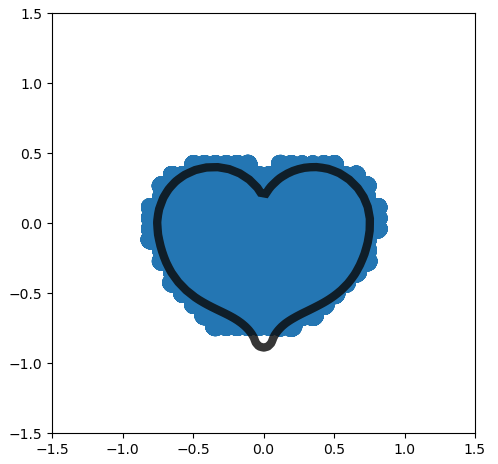}}%
\caption{Shape estimation with dropout neural network. (a) shows the training data. 
(b) shows the estimated mean SDF results. The black heart curve shows the ground-truth obstacle boundary, while colored regions are level-sets of the SDF estimate. The white region denotes the estimated obstacle boundary. The blue (resp. red) region denotes negative (resp. positive) signed distance. In (c), the variance of the SDF estimate is shown. In (d), we plot the estimated unsafe region with high probability, where $\mathbb{P}(\tilde{\varphi} \leq 0) = 0.95$.}
\label{fig:sdf_result}
\vspace*{-3ex}
\end{figure}
%

The robot is equipped with a LiDAR scanner with a $270^\circ$ field of view, $200$ rays per scan, $3$ meter range, and zero-mean Gaussian measurement noise with standard deviation $\sigma \in \{0.01, 0.02, 0.05 \}$. The LiDAR scans are used to estimate the unsafe regions $\calO_i$ in the environment and construct a CBF constraint for each. We rely on the concept of signed distance function (SDF) (e.g. Fig.~\ref{fig:1b}) to describe each $\calO_i$. The SDF function $\varphi_i: \mathbb{R}^2 \mapsto \mathbb{R}$ of set $\calO_i \subseteq \bbR^2$ is
\begin{equation}
\varphi_i(\bfy) := \begin{cases}
 -d(\bfy,\partial \calO_i), & \bfy \in \calO_i, \\
 \phantom{-} d(\bfy,\partial \calO_i), &  \bfy \notin \calO_i,
 \end{cases}
\end{equation} 
where $d$ denotes the Euclidean distance from a point $\bfy \in \mathbb{R}^2$ and the set boundary $\partial \calO_i$.
%
%
We employ incremental training with replay memory (ITRM)~\cite[Sec. IV]{Long_learningcbf_ral21} to estimate an SDF $\varphi_i$ for each $\calO_i$ from the LiDAR measurements. We use a $4$-layer fully-connected neural network with parameters $\bftheta$ and dropout layers to yield $\tilde{\varphi}_i(\bfy;\bftheta)$ with dropout rate $0.05$ applied to each $512$-neuron hidden layer.
Given $\bfy \in \mathbb{R}^2$, we obtain the predictive SDF mean $\hat{\varphi}_i(\bfy)$ and standard deviation $\hat{\sigma}_i(\bfy)$ by Monte-Carlo estimation with $T = 20$ stochastic forward passes through the dropout neural network model. When the TurtleBot moves along a circle of radius $2$ while the object is placed at the center, we measure the accuracy of the online SDF method using the empirical SDF error,
%
$\mathcal{E}_i = \frac{1}{m}\sum_{j=1}^{m} |\hat{\varphi}_i(\bfy_{j}) |$,  
%
where $\{\bfy_j\}_{j=1}^{m}$ are $m = 500$ points uniformly sampled on the surface of the object. 
In Fig.~\ref{fig:sdf_result}, we show the SDF estimation with measurement noise $\sigma = 0.01$.

Since we deal with system dynamics with relative degree one, one can verify~\cite{Chen2021BackupCB} that the SDF is a valid CBF. Let $\bfz = [x, y] \in \calZ \subset \mathbb{R}^2$ be the position part of $\bfx$. To account for the fact that the robot body is not a point mass, we subtract the robot radius $\rho = 0.177$ from each SDF estimate when defining each mean CBF: $\tilde{h}_i(\bfx) = \tilde{\varphi}_i(\bfz;\bftheta) - \rho$. For variance $K_h(\bfx,\bfx)$ in Sec.~\ref{sec:prob_control}, we set $K_{h}^{i}(\bfx,\bfx) = \hat{\sigma}_i^2(\bfz)$. We also take $\nabla \tilde{h}_i(\bfx) = \nabla \tilde{\varphi}_i(\bfz; \bftheta)$ and compute $\calH_{\bfx, \bfx'}K^i_h(\bfx, \bfx')$ by Monte-Carlo estimation using double back-propagation. We set the worst case error bounds $e_h(\bfx)$, $e_{\nabla h}(\bfx)$ in Sec.~\ref{sec:robust_control} as the $99.99\%$ confidence bounds of a Gaussian random variable with standard deviation $\hat{\sigma}_i(\bfz)$. If the robot observes multiple obstacles in the environment, we compute multiple CBFs $\tilde{h}_i(\bfx)$ and their corresponding uncertainty, and add multiple CBCs to \eqref{eq:QP_origin}, \eqref{eq:gaussian_socp_form}, \eqref{eq:worst_case_SOCP_formulation} for safe control synthesis.


\subsection{Safe Navigation}
\label{sec: safe_navigation}

\begin{table}[t]
\centering
\caption{Success rate of the navigation tasks in $100$ realizations ($10$ realizations for each of the $10$ different environments) using the Probabilistic CLF-CBF-SOCP, Robust CLF-CBF-SOCP, and the original CLF-CBF-QP frameworks for different LiDAR measurement noise levels $\sigma$.}
\resizebox{\linewidth}{!}{
\begin{tabular}{ |l|c|c|c|c|c| }
\hline
\multirow{2}{1cm}{LiDAR Noise $\sigma$} & \multirow{2}{1.4cm}{QP Success Rate} & \multicolumn{3}{c|}{Probabilistic Success Rate} & \multirow{2}{1.5cm}{Robust Success Rate}\\\cline{3-5}
&  & $p = 0.2$ & $p = 0.4$ & $p = 0.8$ &  \\
\hline
$0.01$ & $0.82$ & $0.98$ & $1.0$ & $1.0$ & $1.0$ \\
$0.02$  & $0.65$ & $0.92$ & $0.97$ & $1.0$ & $1.0$\\ 
$0.05$   &  $0.37$ & $0.72$ & $0.89$ & $0.96$ & $1.0$ \\\hline 
\end{tabular}}
\label{table:noise_diff}
\caption{Fr\'echet distance between the reference path and the robot trajectories generated by the Probabilistic CLF-CBF-SOCP, Robust CLF-CBF-SOCP, and the CLF-CBF-QP controllers (smaller values indicate larger trajectory similarity, the value in the parentheses indicates the success rates while values without parentheses indicate the success rate is $1$}, and N/A indicates the robot collides with obstacles in all $10$ realizations).
\resizebox{\linewidth}{!}{
\begin{tabular}{ |l|c|c|c|c|c| }
\hline
\multirow{2}{*}{Env} & \multirow{2}{*}{QP} & \multicolumn{3}{c|}{Probabilistic} & \multirow{2}{*}{Robust} \\\cline{3-5}
&  & $p = 0.2$ & $p = 0.4$ & $p = 0.8$ &  \\
\hline
1 & $0.337$ & $0.338$ & $0.343$ & $0.363$ & $0.357$\\
2  & $0.378$& $0.408$ & $0.404$ & $0.432$ & $0.485$ \\ 
3   & $0.372$ & $0.398$ & $0.412$ & $0.457$ & $0.538$\\
4 & $0.416$&  $0.438$ & $0.427$ & $0.473$ & $0.515$ \\
5  & $0.395$ &  $0.418$ & $0.412$ & $0.483$ & $0.572$\\
6   & $0.385 \: (0.8)$ & $0.371$ & $0.378$ & $0.392$ & $0.424$\\
7 & $0.462 \: (0.5)$ & $0.502$ & $0.546$ & $0.593$ & $0.737$\\
8  &  $0.535 \: (0.2)$ & $0.588$ & $0.612$ & $0.673$ & $0.814$\\
9  &  N/A & $0.756 \: (0.8)$& $0.887 \: (0.9)$ & $0.926$ & $1.016$\\
10  &  N/A &$0.905 \: (0.4)$ & $0.937 \: (0.8)$ & $1.046$ & $1.224$\\
\hline 
\end{tabular}}
\label{table:traj_diff}
\vspace*{-3ex}
\end{table}


\begin{figure}[t]
  \centering
  \subcaptionbox{Noise: $\sigma = 0.01$\label{fig:3a}}{\includegraphics[width=0.4\linewidth]{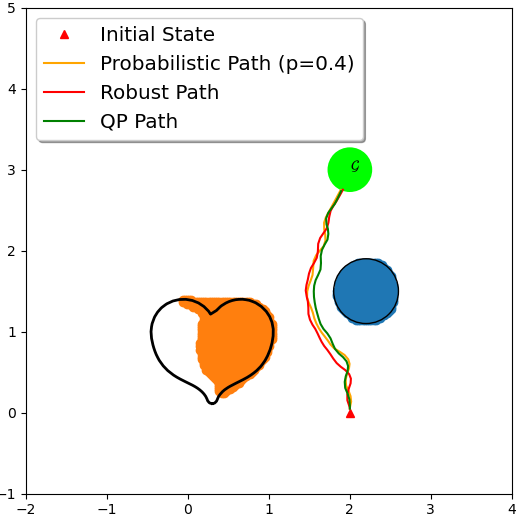}}%
  \subcaptionbox{Noise: $\sigma = 0.02$\label{fig:3b}}{\includegraphics[width=0.4\linewidth]{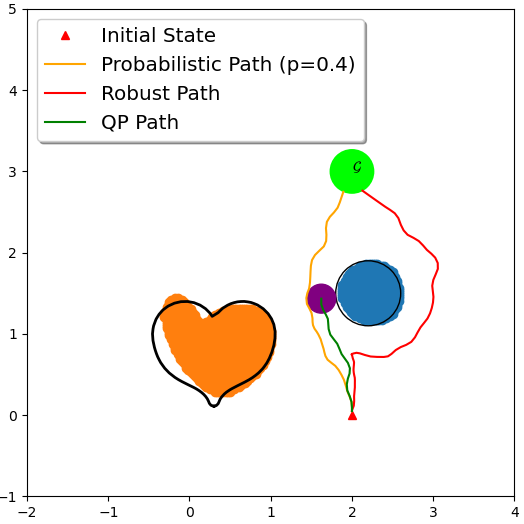}} \\ 
  \centering
  \subcaptionbox{Safe Trajectory Tracking\label{fig:3c}}{\includegraphics[width=0.7\linewidth]{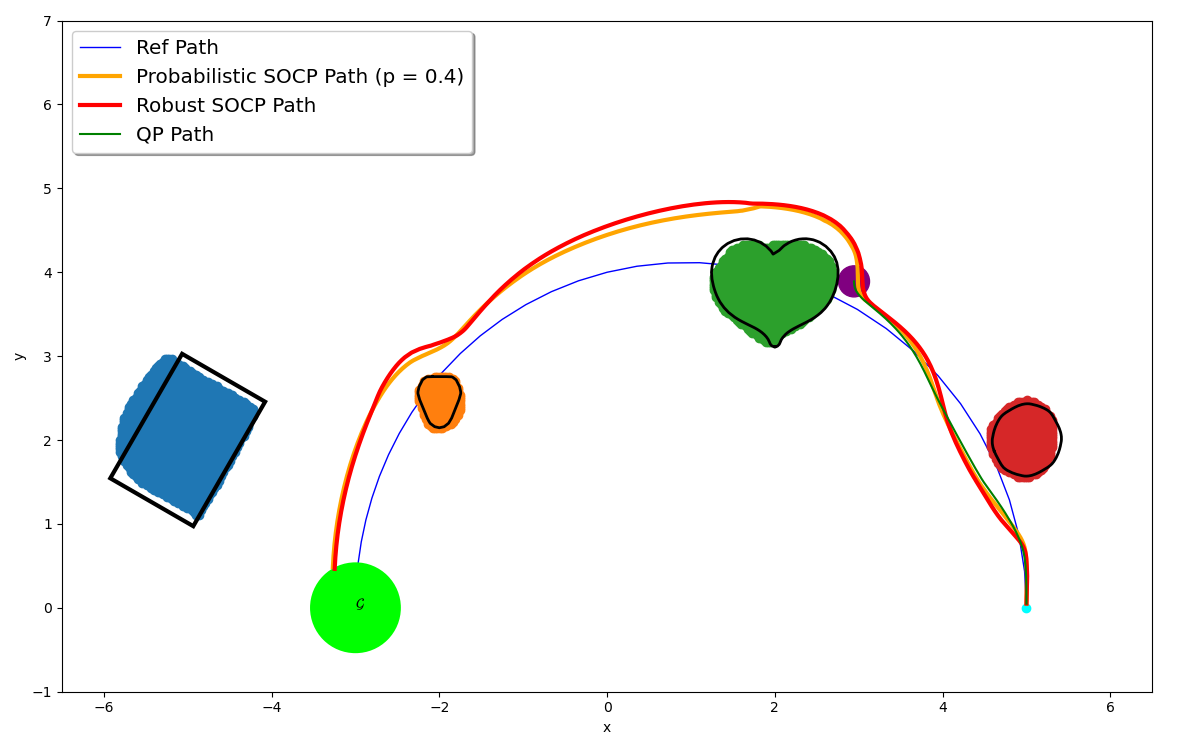}}
  \caption{Performance comparison among the three controllers. Ground-truth obstacle surfaces are shown as black curves. The mean of the estimated obstacles, obtained after the whole path is traversed by the probabilistic CLF-CBF-SOCP controller are shown in different colors (red, green, orange, blue). The trajectories generated by the probabilistic and robust CLF-CBF-SOCP controllers are in red and blue, respectively, while the CLF-CBF-QP trajectory is in green. The starting point is cyan and the goal region is a light-green disk. The robot (purple disk in (b) and (c)), controlled by the CLF-CBF-QP controller, collides with obstacles and does not reach the goal. In (a) and (b), we compare the controller performance under different LiDAR noise level for a same environment. In (a), the results are collected under LiDAR measurement noise $\sigma = 0.01$. In (b), the results are collected under LiDAR noise $\sigma = 0.02$. In (c), the trajectory tracking results of environment $8$ is shown and the reference path is shown in blue.}
  \label{fig:compare_result}
  \vspace*{-3ex}
\end{figure}

Our main experiments demonstrate safe navigation and safe trajectory tracking using the proposed probabilistic~\eqref{eq:gaussian_socp_form} and robust~\eqref{eq:worst_case_SOCP_formulation} CLF-CBF-SOCP formulations, utilizing the dynamics estimates from Sec. VI.A and the online CBF estimates from Sec. VI.B. To emphasize the importance of accounting for estimation errors, we also implement the original CLF-CBF-QP controller~\eqref{eq:QP_origin}, which assumes the estimated barrier functions and system dynamics are accurate (i.e., uses the mean values from the dropout-network estimation as the true values). In all three controllers, we set $L(\bfx) = \diag([0, 10, 3])$ and $\underline{\tilde{\bfk}}(\bfx) = [1, v_{\max}, 0]^\top$ where $v_{\max} = 0.65$ is the maximum linear velocity for the TurtleBot. The remaining parameters were $\lambda = 1000$, $\alpha_V(V(\bfx)) = 2V(\bfx)$, and $\alpha_h(h_i(\bfx)) = h_i(\bfx)$.

In the first set of experiments (Fig.~\ref{fig:3a} and Fig.~\ref{fig:3b}), we demonstrate safe navigation to a goal point with a CLF candidate $V(\bfx) = (x-2)^2 + (y-3)^2$. In Fig.~\ref{fig:3a}, when the LiDAR noise level is low, the robot controlled by all three controllers succeeds to reach the goal region and the SOCP formulations are slightly more conservative than the QP formulation. In Fig.~\ref{fig:3b}, the LiDAR noise level increases to $\sigma = 0.02$ and we can observe major differences among the paths generated by the three controllers. This is because the estimated variance and error bounds of the barrier function and its gradient increase with the increase of the LiDAR noise. The robot controlled by the CLF-CBF-QP controller collides with an obstacle,  while the robot controlled by  probabilistic or robust SOCP controller succeeds in avoiding obstacles. Importantly, the robot controlled by the robust SOCP controller switches to bypass the round obstacle from the right because controller cannot find a feasible path on the left with larger error bounds on estimated barrier functions. 

In the following set of experiments, we consider the problem of safe trajectory tracking using the approach in \cite[Sec. VI]{Long_learningcbf_ral21} to construct a valid CLF $V(\bfx)$ for path following. 
In Table~\ref{table:noise_diff}, we report the success rate of the trajectory tracking task using the proposed formulations and the original QP framework under different measurement noises. As the noise increases, the success rate of the CLF-CBF-QP controller decreases rapidly, while the success rate of the probabilistic framework with high $p$ and the robust framework stays high. 

In Fig.~\ref{fig:3c}, we show one realization in environment $8$ (with $\sigma = 0.01$), where the CLF-CBF-QP controller fails to avoid obstacles because it does not consider the errors in $\textit{CBC}(\bfx,\ubfu)$, while the proposed frameworks guarantee safety. 
%
%
When there is an obstacle near or on the reference path, the robot controlled by the robust SOCP controller stays furthest away, while the probabilistic SOCP controller also guarantees the robot stays further away from the obstacles than the robot controlled by the CLF-CBF-QP controller.

%
%


In Table~\ref{table:traj_diff}, we show quantitative results using the Fr\'echet distance ~\cite[Sec. VI]{Long_learningcbf_ral21} as the metric to measure trajectory similarities. The distance value is computed by averaging the successful realizations in each environment, and the LiDAR noise is set to be $\sigma = 0.02$ in this set of experiments. We see that the robust CLF-CBF-SOCP controller is most conservative as it has the largest Fr\'echet distance values while the probabilistic CLF-CBF-SOCP controller is less conservative if we set the user-specified risk tolerance $p = 0.8$. By lowering the risk tolerance value ($p = 0.2/0.4$), the robot with the probabilistic controller follows the reference path better while facing a higher risk of collision. A qualitative result is shown in Fig.~\ref{fig:1b}, where larger $p$ values indicates higher probability of being safe for the robot. The trajectory generated by the CLF-CBF-QP controller has the smallest Fr\'echet distance values, but fails in several environments. 

Finally, to demonstrate the efficiency of the proposed formulations, we compare the average time needed for solving the QP, probabilistic SOCP, and robust SOCP formulations per control synthesis along the trajectory tracking task. All optimization problems are solved using the Embedded Conic Solver in CVXPY~\cite{cvxpy} with an Intel i7 9700K CPU. The time needed for solving one QP instance is $0.00863s$ while the times needed for solving the proposed probabilistic and robust SOCPs are $0.0109s$ and $0.0122s$. As expected, our SOCP formulations require slightly more time than the original QP but are still suitable for online robot navigation.


\section{Conclusion}

We considered the problem of enforcing safety and stability of unknown robot systems operating in unknown environments. We showed that accounting for either Gaussian or worst-case error bounds in the system dynamics and safety constraints leads to a novel CLF-CBF-SOCP formulation for control synthesis. We validated our formulations in autonomous navigation tasks, simulating a ground robot in several unknown environments. Some drawbacks of our formulations include that large model error bounds may lead to infeasibility of the robust SOCP, and that the assumption that system dynamics and barrier functions are GPs may not be true in practice. Future work will implement the proposed formulations on a real robot, consider object category pre-training of the SDF neural network, and explore adaptive techniques for safe control synthesis given varying uncertainty levels and robot objectives.





%
%

%
%
\bibliographystyle{ieeetr}
\bibliography{ref.bib}

\begin{thebibliography}{10}

\bibitem{lamport1977safety}
L.~Lamport, ``Proving the correctness of multiprocess programs,'' {\em IEEE
  Trans. on Software Engg.}, vol.~SE-3, no.~2, pp.~125--143, 1977.

\bibitem{Artstein1983StabilizationWR}
Z.~Artstein, ``Stabilization with relaxed controls,'' {\em Nonlinear
  Analysis-theory Methods \& Applications}, vol.~7, pp.~1163--1173, 1983.

\bibitem{SONTAG1989117}
E.~D. Sontag, ``{A ‘universal’ construction of Artstein's theorem on
  nonlinear stabilization},'' {\em Systems \& Control Letters}, vol.~13, no.~2,
  pp.~117--123, 1989.

\bibitem{barrier-certificate}
S.~{Prajna}, ``Barrier certificates for nonlinear model validation,'' in {\em
  Conference on Decision and Control}, pp.~2884--2889, 2003.

\bibitem{barrier-hybrid}
S.~Prajna and A.~Jadbabaie, ``Safety verification of hybrid systems using
  barrier certificates,'' in {\em Hybrid Systems: Computation and Control},
  pp.~477--492, Springer Berlin Heidelberg, 2004.

\bibitem{wieland2007}
P.~Wieland and F.~Allg{\"o}wer, ``Constructive safety using control barrier
  functions,'' in {\em IFAC Proceedings Volumes}, pp.~462--467, 2007.

\bibitem{ames2016control}
A.~Ames, X.~Xu, J.~Grizzle, and P.~Tabuada, ``Control barrier function based
  quadratic programs for safety critical systems,'' {\em IEEE Transactions on
  Automatic Control}, vol.~62, no.~8, pp.~3861--3876, 2016.

\bibitem{nguyen2016acc}
Q.~Nguyen and K.~Sreenath, ``Exponential control barrier functions for
  enforcing high relative-degree safety-critical constraints,'' in {\em
  American Control Conference}, pp.~322--328, 2016.

\bibitem{cbf}
A.~Ames, S.~Coogan, M.~Egerstedt, G.~Notomista, K.~Sreenath, and P.~Tabuada,
  ``Control barrier functions: Theory and applications,'' in {\em European
  Control Conference}, pp.~3420--3431, 2019.

\bibitem{Wang2017SafeCM}
L.~Wang, A.~D. Ames, and M.~Egerstedt, ``Safe certificate-based maneuvers for
  teams of quadrotors using differential flatness,'' {\em IEEE International
  Conference on Robotics and Automation}, pp.~3293--3298, 2017.

\bibitem{nguyen2016cdc}
Q.~Nguyen, A.~Hereid, J.~W. Grizzle, A.~D. Ames, and K.~Sreenath, ``3d dynamic
  walking on stepping stones with control barrier functions,'' in {\em IEEE
  Conference on Decision and Control}, pp.~827--834, 2016.

\bibitem{xu2017realizing}
X.~Xu, T.~Waters, D.~Pickem, P.~Glotfelter, M.~Egerstedt, P.~Tabuada, J.~W.
  Grizzle, and A.~D. Ames, ``Realizing simultaneous lane keeping and adaptive
  speed regulation on accessible mobile robot testbeds,'' in {\em IEEE
  Conference on Control Technology and Applications}, pp.~1769--1775, 2017.

\bibitem{CantelliSuiCD}
F.~P. Cantelli, ``Sui confini della probabilit{\`a},'' {\em Atti del Congresso
  Internazionale dei Matematici}, vol.~6, pp.~47--60, 1929.

\bibitem{jankovic_robust_2018}
M.~Jankovic, ``Robust control barrier functions for constrained stabilization
  of nonlinear systems,'' {\em Automatica}, vol.~96, p.~359, 2018.

\bibitem{yousef2019cdc}
Y.~Emam, P.~Glotfelter, and M.~Egerstedt, ``Robust barrier functions for a
  fully autonomous, remotely accessible swarm-robotics testbed,'' in {\em IEEE
  Conference on Decision and Control}, pp.~3984--3990, 2019.

\bibitem{andrew2019acc}
A.~Clark, ``Control barrier functions for complete and incomplete information
  stochastic systems,'' in {\em ACC}, pp.~2928--2935, 2019.

\bibitem{Nguyen2021}
Q.~Nguyen and K.~Sreenath, ``Robust safety-critical control for dynamic
  robotics,'' {\em IEEE Transactions on Automatic Control}, 2021.

\bibitem{lukas_2020_TCST}
L.~Hewing, J.~Kabzan, and M.~N. Zeilinger, ``Cautious model predictive control
  using gaussian process regression,'' {\em IEEE Transactions on Control
  Systems Technology}, vol.~28, no.~6, pp.~2736--2743, 2020.

\bibitem{Ahmadi2020RiskSensitivePP}
M.~Ahmadi, X.~Xiong, and A.~D. Ames, ``Risk-averse control via {CV}a{R} barrier
  functions: Application to bipedal robot locomotion,'' {\em IEEE Control
  Systems Letters}, vol.~6, pp.~878--883, 2022.

\bibitem{Alcan_DDP_RAL}
G.~Alcan and V.~Kyrki, ``Differential dynamic programming with nonlinear safety
  constraints under system uncertainties,'' {\em IEEE Robotics and Automation
  Letters}, vol.~7, no.~2, pp.~1760--1767, 2022.

\bibitem{Hassan_DBaS_RAL}
H.~Almubarak, K.~Stachowicz, N.~Sadegh, and E.~A. Theodorou, ``Safety embedded
  differential dynamic programming using discrete barrier states,'' {\em IEEE
  RAL}, vol.~7, no.~2, pp.~2755--2762, 2022.

\bibitem{romdlony2016cdc}
M.~Z. Romdlony and B.~Jayawardhana, ``On the new notion of input-to-state
  safety,'' in {\em IEEE CDC}, pp.~6403--6409, 2016.

\bibitem{Kolathaya2019issf}
S.~Kolathaya and A.~D. Ames, ``Input-to-state safety with control barrier
  functions,'' {\em IEEE CSL}, vol.~3, no.~1, pp.~108--113, 2019.

\bibitem{alan2021safe}
A.~Alan, A.~J. Taylor, C.~R. He, G.~Orosz, and A.~D. Ames, ``Safe controller
  synthesis with tunable input-to-state safe control barrier functions,'' {\em
  IEEE Control Systems Letters}, vol.~6, pp.~908--913, 2022.

\bibitem{cosner2021measurement}
R.~K. Cosner, A.~W. Singletary, A.~J. Taylor, T.~G. Molnar, K.~L. Bouman, and
  A.~D. Ames, ``Measurement-robust control barrier functions: Certainty in
  safety with uncertainty in state,'' {\em arXiv}, vol.~abs/2104.14030, 2021.

\bibitem{srinivasan2020synthesis}
M.~Srinivasan, A.~Dabholkar, S.~Coogan, and P.~Vela, ``Synthesis of control
  barrier functions using a supervised machine learning approach,'' {\em
  IEEE/RSJ IROS}, pp.~7139--7145, 2020.

\bibitem{zhang2021adversarially}
T.~T. Zhang, S.~Tu, N.~M. Boffi, J.-J.~E. Slotine, and N.~Matni,
  ``Adversarially robust stability certificates can be sample-efficient,'' {\em
  arXiv preprint arXiv:2112.10690}, 2021.

\bibitem{dhiman2020control}
V.~Dhiman$^*$, M.~J. Khojasteh$^*$, M.~Franceschetti, and N.~Atanasov,
  ``Control barriers in bayesian learning of system dynamics,'' {\em IEEE
  Transactions on Automatic Control}, 2021.

\bibitem{Long_learningcbf_ral21}
K.~Long, C.~Qian, J.~Cortés, and N.~Atanasov, ``Learning barrier functions
  with memory for robust safe navigation,'' {\em IEEE Robotics and Automation
  Letters}, vol.~6, no.~3, pp.~4931--4938, 2021.

\bibitem{alizadeh2003second}
F.~Alizadeh and D.~Goldfarb, ``Second-order cone programming,'' {\em
  Mathematical programming}, vol.~95, no.~1, pp.~3--51, 2003.

\bibitem{coumans2020}
E.~Coumans and Y.~Bai, ``{PyBullet, a Python module for physics simulation for
  games, robotics and machine learning}.'' \url{http://pybullet.org}, 2016.

\bibitem{chen2018neuralode}
R.~T.~Q. Chen, Y.~Rubanova, J.~Bettencourt, and D.~Duvenaud, ``Neural ordinary
  differential equations,'' {\em Advances in Neural Information Processing
  Systems}, 2018.

\bibitem{2015yarin}
Y.~Gal and Z.~Ghahramani, ``Dropout as a bayesian approximation: Representing
  model uncertainty in deep learning,'' in {\em International Conference on
  Machine Learning}, vol.~48, pp.~1050--1059, 2016.

\bibitem{Chen2021BackupCB}
Y.~Chen, M.~Jankovic, M.~A. Santillo, and A.~Ames, ``Backup control barrier
  functions: Formulation and comparative study,'' {\em arXiv},
  vol.~abs/2104.11332, 2021.

\bibitem{cvxpy}
S.~Diamond and S.~Boyd, ``{CVXPY}: A {P}ython-embedded modeling language for
  convex optimization,'' {\em Journal of Machine Learning Research}, 2016.

\end{thebibliography}


\end{document}